\documentclass[letterpaper]{article} 
\usepackage{aaai23}  
\usepackage{times}  
\usepackage{helvet}  
\usepackage{courier}  
\usepackage[hyphens]{url}  
\usepackage{graphicx} 
\usepackage{natbib}  
\usepackage{caption} 
\usepackage{algorithm}
\usepackage{algorithmic}

\usepackage{url}            
\usepackage{booktabs}       
\usepackage{amsfonts}       
\usepackage{nicefrac}       
\usepackage{microtype}      
\usepackage{xcolor}         
\usepackage{amsmath}
\usepackage{amssymb}
\usepackage{mathtools}
\usepackage{amsthm}
\usepackage[utf8]{inputenc} 
\usepackage[T1]{fontenc}    

\usepackage{color}
\usepackage{caption}
\usepackage{subfigure}
\usepackage{graphicx}
\usepackage{multirow}
\usepackage{url}
\usepackage{amssymb}

\usepackage{amsmath}               
  {
      \theoremstyle{plain}
      \newtheorem{assumption}{Assumption}
      \newtheorem{theorem}{Theorem}
      \newtheorem{definition}{Definition}
      \newtheorem{lemma}{Lemma}
  }

\usepackage{newfloat}
\usepackage{listings}
\DeclareCaptionStyle{ruled}{labelfont=normalfont,labelsep=colon,strut=off} 
\floatstyle{ruled}
\newfloat{listing}{tb}{lst}{}
\floatname{listing}{Listing}
\title{Estimating Treatment Effects from Irregular Time Series Observations \\ with Hidden Confounders}
\author{
    Defu Cao\equalcontrib\textsuperscript{\rm 1}, James Enouen \equalcontrib \textsuperscript{\rm 1}, Yujing Wang \textsuperscript{\rm 2}, Xiangchen Song \textsuperscript{\rm 3}, \\ Chuizheng Meng \textsuperscript{\rm 1}, Hao Niu \textsuperscript{\rm 4}, Yan Liu \textsuperscript{\rm 1}\\
}
\affiliations{
    \textsuperscript{\rm 1}University of Southern California\\
    \textsuperscript{\rm 2}Peking University\\ 

    \textsuperscript{\rm 3}Carnegie Mellon University\\ \textsuperscript{\rm 4}KDDI Research, Inc.  \\

     \{defucao, enouen, chuizhem, yanliu.cs\}@usc.edu, yujwang@pku.edu.cn\\
 xiangchensong@cmu.edu, ha-niu@kddi.com
}

\usepackage{bibentry}

\begin{document}

\maketitle

\begin{abstract}

Causal analysis for time series data, in particular estimating individualized treatment effect (ITE), is a key task in many real-world applications, such as finance, retail, healthcare,  etc.  
{ Real-world time series can include large-scale, irregular, and intermittent time series observations, raising significant challenges to existing work attempting to estimate treatment effects.}
Specifically, the existence of hidden confounders can lead to  biased treatment estimates and complicate the causal inference process. In particular, anomaly hidden confounders which exceed the typical range can lead to high variance estimates. Moreover,  in continuous time settings with irregular samples, it is challenging to directly handle the dynamics of causality. In this paper, we leverage recent advances in Lipschitz regularization and neural controlled differential equations (CDE)  to develop an effective and scalable solution, namely LipCDE, to address the above challenges. LipCDE can directly model the dynamic causal relationships between historical data and outcomes with irregular samples by considering the boundary of hidden confounders given by Lipschitz-constrained neural networks. Furthermore, we conduct extensive experiments on both synthetic and real-world datasets to demonstrate the effectiveness and scalability of LipCDE. 
\end{abstract}

\setcounter{secnumdepth}{2}
\section{Introduction}
\label{submission}

Estimating individualized treatment effects (ITE) for time series data, which makes predictions about causal effects of actions~\cite{zhangcounterfactual}, is one key task in many domains, including marketing~\cite{brodersen2015inferring, abadie2010synthetic}, education~\cite{mandel2014offline}, healthcare~\cite{kuzmanovic2021deconfounding}, etc. 
However, the existence of confounders
can introduce bias into the estimation \cite{simpson1951interpretation, pearl2000models}. For example, in finance applications, multi-factor investing strategies can give investors a deeper understanding of the risk drivers underlying a portfolio. The unobserved factors (i.e., hidden confounders), which typically happen at irregular time stamps and are not reflected in finance system records or are difficult to observe, could bring bias by influencing both interventions and stock returns. The reason is that even a small number of existing factors (such as Small Minus Big and High Minus Low) could significantly explain the cross-section of stock returns ~\cite{d2021evolving}.  
If we can simulate such hidden confounders within a reasonable range,
we are able to obtain  treatment estimates with reduced bias and  variance by making appropriate impact assumptions on the relationship between treatments and outcomes ~\cite{wang2019blessings}.

Estimating ITE is an extremely challenging task in continuous time settings with hidden confounders.  
First, estimating treatment
effects in large-scale irregular and sparse time series still has
considerable room for improvement as previous works fail
to consider the continuous time setting, where it is difficult to handle the dynamic behavior and complex interactions of covariates and treatments~\cite{gao2021causal}. 
Second, hidden confounders' values generated by randomness and noise can introduce high variance and undesirable explanations. 
For example, in healthcare applications, according to domain knowledge of drug resistance, the response to single-agent immune-checkpoint inhibitors (ICI) in uremic patients ranged from $15\%$ to $31\%$~\cite{zibelman2016emerging}. Consequently, when left unconsidered, drug resistance will introduce biased estimates of  treatment effects. Furthermore,
any substitute  confounders generated by data-driven methods with an impact on outcomes over $31\%$ can lead to high variance. 

Recently, there have been several attempts to address these challenges. To model hidden confounders over time, ~\cite{bahadori2021debiasing} introduce a new causal prior graph for the confounding information and concept completeness to improve the interpretability of prediction models; ~\cite{pmlr-v139-mastakouri21a} study the identification of direct and indirect causes for causal feature selection in time series solely based on observational data. Deconfounding-based models ~\cite{hatt2021sequential, bica2020time} use latent variables given by their factor model as substitutes for the hidden confounders to render the assigned treatments conditionally independent. However, existing works either cannot handle irregular time series~\cite{bahadori2021debiasing, pmlr-v139-mastakouri21a}, or have strong assumptions  ~\cite{hatt2021sequential, bica2020time}. 
Furthermore, the range of hidden confounders generated by previous data-driven works is possibly  unjustifiable, which will distort (obscure or augment) the true causal relationship between treatments and outcomes.

In this work, we consider the task of estimating treatment effects under continuous time settings with multi-cause hidden confounders (which affect multiple treatments and the outcome).
To tackle the above two challenges, we propose a novel Lipschitz regularized neural controlled differential equation (LipCDE) model for estimation by obtaining the constrained time-varying hidden confounders. 
Specifically, LipCDE first infers the interrelationship of hidden confounders on treatment by estimating the boundary of hidden confounders: we decompose the historical covariates into low-frequency components and high-frequency components in the spectral domain.  Then we use Lipschitz regularization~\cite{araujo2021lipschitz} on the decomposition to get the latent representation.  Afterward, we model the historical trajectories with neural CDE using sparse numerical solvers, which is one of the most suitable methods for large-scale problems under the continuous time setting~\cite{frohlich2019scalable}. 
In this way, we can explicitly model the observed irregular sequential data as a process evolving continuously in time with a dynamic causal relationship to equip the LipCDE with interpretability. In the outcome model, we re-weight the population of all participating patients and balance the representation via applying the inverse probability of treatment weighting (IPTW) strategy ~\cite{lim2018forecasting}.

 In this paper, we conduct extensive experiments on  both  simulated and real-world datasets. 
Experimental results show that LipCDE outperforms other state-of-the-art estimating treatment effect approaches. From a qualitative perspective,  
experiments show that LipCDE is in agreement with the true underlying hidden confounders in simulated environments, which can effectively eliminate bias in causal models ~\cite{pearl2000models}.
In addition, the average  RMSE of TSD~\cite{bica2020time} and SeqConf~\cite{hatt2021sequential}  on MIMIC-III's blood pressure outcome and COVID-19 datasets decreases by 28.7\%  and 32.3\%,  respectively.
To the best of our knowledge, this is the first complete estimating treatment effects model that considers both the boundary of hidden confounders and the continuous time setting.

We summarize the main \textbf{contributions} as follows: 
\begin{itemize}
    \item LipCDE utilizes a convolutional operation with Lipschitz regularization on the spectral domain and neural controlled differential equation from observed data to obtain hidden confounders, which are bounded to reduce the high variance of treatment effect estimation. 
    \item LipCDE can fully use the information from observed data and dynamic time intervals, allowing the continuous inclusion of input interventions and supporting irregularly sampled time series.
    \item Sufficient experiments demonstrate the effectiveness of LipCDE in estimating treatment effects on both synthetic and real-world datasets. Particularly, experiments on  MIMIC-III and COVID-19 demonstrate the potential of LipCDE for healthcare applications in personalized medical recommendation. 
\end{itemize}

\section{Related Work}

\textbf{Treatment effects learning in the static setting.}
In recent years, there has been a significant increase in interest in the study of causal inference accomplished through representational learning~\cite{kallus2018causal,curth2021inductive}. ~\cite{johansson2016learning}  proposes to take advantage of the multiple processing methods assigned in a static environment. \cite{shalit2017estimating} show that balancing the representational distributions of the treatment and control groups can help upper limits of error for counterfactual outcome estimates. However, these approaches rely on the strong ignorability assumption, which ignores the influence of implicitly hidden confounders.  Many works focus on  relaxing  such assumptions with the consideration of hidden confounders including domain adversarial training~\cite{berrevoets2020organite,curth2021inductive}. \cite{guo2020counterfactual} and \cite{guo2020learning} propose to unravel the patterns of hidden confounders from the network structure and observed features by learning the representations of hidden confounders and using the representations for potential outcome prediction. \cite{wang2019blessings} propose to estimate confounding factors in a static setting using a latent factor model and then infer potential outcomes using bias adjustment. Nevertheless, such works fail to take advantage of the dynamic evolution of the observed variables and the inter-individual relationships which are present in the time-dynamic setting. 

\textbf{Treatment effects learning in the dynamic setting without hidden confounders.} There are many related previous works estimating treatment effects in dynamic settings
including g-computation formula, g-estimation of structural nested mean models \cite{hernan2010causal}, IPTW in marginal structural models (MSMs) \cite{robins2009estimation},  and recurrent marginal structural networks (RMSNs) \cite{lim2018forecasting}, CRN ~\cite{Bica2020Estimating} etc.
In addition, Gaussian processes~\cite{schulam2017reliable} and bayesian nonparametrics~\cite{roy2017bayesian} have been tailored to estimate treatment response in a continuous time setting in order to incorporate non-deterministic quantification. Besides, ~\cite{soleimani2017treatment} relies on regularization to decompose the observed data into shared and signal-specific components in treatment response curves from multivariate longitudinal data. However, those models still need constraint methods to guarantee the posterior consistency of the sub-component modules and cannot directly model the dynamic causal relationship between different time intervals. While ~\cite{seedat2022continuous, de2022predicting} directly model the dynamic causal relationship, they make a strong assumption with no hidden confounders, which 
does not have the flexibility to be applied to all real-world scenarios.

\textbf{Treatment effects learning in the dynamic setting with hidden confounders.} Rather than making strong ignorability assumptions,  \cite{pearl2012measurement} and \cite{kuroki2014measurement} theoretically prove that observed proxy variables can be used to capture hidden confounders and estimate treatment effects. \cite{veitch2020adapting} use network information as a proxy variable to mitigate confounding bias without utilizing the characteristics of the instances.
TSD \cite{bica2020time} introduces recurrent neural networks in the factor model to estimate the dynamics of confounders. In a similar vein, \cite{hatt2021sequential} propose a sequential deconfounder to infer hidden confounders  by using Gaussian process latent variable model and DTA~\cite{kuzmanovic2021deconfounding} estimates treatment effects under dynamic setting using observed data as noisy proxies. Besides, DSW \cite{RefWorks:doc:6170aa6bc9e77c00014e305e}  infers the hidden confounders by using a deep recursive weighted neural network that combines current treatment assignment and historical information. DNDC \cite{ma2021deconfounding} aims to learn how hidden confounders behave over time by using current network observation data and historical information. However, previous works have not bounded confounders leading to high variance estimates when the data-driven approach produces anomaly confounders which have exceeded the impact constraint over treatments and outcomes. 


\section{Problem Setup}

\subsection{Estimating Treatment Effects Task}  
Here we define the problem of estimating treatment effects from irregular time series observations formally: observational data for each patient $i$ at irregular time steps $t_{0}^i<\cdots<t_{m_i}^i$ for some $m_i\in\mathbb{N}$.
We have observed covariates $X^i = [X_{t_0}^i,X_{t_1}^i,  \ldots, X_{t_{m_i}}^i] \in \mathcal{X}_t$ and corresponding treatments $A^i = [A_{t_0}^i,A_{t_1}^i,  \ldots, A_{t_{m_i}}^i] \in \mathcal{A}_t$, and $a_{t_k}$
is the set of all $j$ possible assigned  treatments at timestep $t_k$.
Additionally, we have hidden confounder variables $Z^i = [Z_{t_0}^i, Z_{t_1}^i,  \ldots, Z_{t_{m_i}}^i] \in \mathcal{Z}_t$.  We omit the patient id $i$ on timestamps unless they are explicitly needed. 
Combining all hidden confounders, observed covariates, and observed treatments, we define the history before time $t_k$ as $H_{ t_k}^i =\{X_{< t_k}^i, A_{< t_k}^i, Z_{< t_k}^i\}$
as the collection of all historical information.

We focus on one-dimensional outcomes $Y^i = [y_{t_0}^i,y_{t_1}^i,  \ldots, y_{t_{m}}^i] \in \mathcal{Y}_t$ and  we will be interested in the final expected outcome $\mathbb{E}[Y^i_{a_t,t_m}|H^i_t, X_{t}^i, A_{t}^i, Z_{t}^i]$, given a specified treatment plan $a$.
In this way, we can define the individual treatment effect (ITE) with historical data as $\tau_t^i = \mathbb{E}[Y^i_{b_t,t_m}|H^i_t, X_{t}^i, A_{t}^i, Z_{t}^i] - \mathbb{E}[Y^i_{a_t,t_m}|H^i_t, X_{t}^i, A_{t}^i, Z_{t}^i]$ for two specified treatments $a$ and $b$.
In practice, we rely on assumptions to be able to estimate $\tau^i_t$ for any possible treatment plan,
which begins at time step $t$ until just before the final patient outcome $Y$ is measured:
\begin{assumption}
\label{Consistency}
\textnormal{Consistency~\cite{lim2018forecasting}.} If $A_{\geq t} = a_{\geq t}$,  then the potential outcomes for following the treatment plan $a_{\geq t}$ is the
same as the observed (factual) outcome $Y_{a_{\geq t}} = Y$.
\end{assumption}

\begin{assumption}
\label{Positivity}
\textnormal{Positivity (Overlap)~\cite{imai2004causal}.} For any patient, if the probability  $P(a_{< t_{m}} , z_{< t_{m}}, x_{\leq t_m} ) \ne 0$ then the probability of assigning treatment:  $P(A_{t_{m}} =a_{t_m}  |a_{< t_{m}} , z_{< t_{m}}, x_{\leq t_m} )> 0 $ for all $a_{t_m}$.
\end{assumption}

Assumption ~\ref{Consistency} and Assumption ~\ref{Positivity} are relatively standard assumptions of causal inference that assume that artificially assigning a treatment has the same impact as if it were naturally assigned and that each treatment has some nonzero probability. 
Additionally, most previous works in the time series domain make the sequential strong ignorability assumption ~\cite{robins2009estimation} that if there are no hidden confounders, for all possible treatments $A_t$, given the historically observed covariates $X_t$, we have: $Y_{a_{\geq t_m}} \perp \!\!\! \perp  A_{t_m}|A_{{< t_{m}}},X_{< t_m}$.
However, this assumption is often untestable due to the presence of hidden confounders in the real-world. Inspired by \cite{wang2019blessings} and  \cite{bica2020time}, {we assume} sequential single strong ignorability in the continuous time setting:

\begin{assumption}
\label{sssi}
\textnormal{Sequential single strong ignorability in continuous time setting.} 
If there exist multi-cause confounders, we have
$Y_{a_{\geq t_m}} \perp \!\!\! \perp  A_{t_m}|X_{t_m}, H_{< t_{m}}$, for all $a_{\geq t_m}$ and all
$j$ possible assigned treatments.
\end{assumption}

Assumption ~\ref{sssi} expands the sequential single strong ignorability assumption from ~\cite{bica2020time} to the continuous time setting.
Thus, only multi-cause hidden confounders exist at every time stamp, having a causal effect on the treatment $A_t$ and potential outcome $Y_t$.
One of our goals is to learn representations of hidden confounders under the line of deconfounding works, which aim to eliminate bias, based on the following theorem:

\begin{theorem}
\label{the1}
If the distribution of the assigned causes $p(a_T )$ can be written as $p(\theta, x_T , z_T , a_T )$, we can obtain sequential ignorable treatment assignment:
\begin{align}
    Y_{a_{\geq t_m}} \perp \!\!\! \perp  A_{t_m} | X_{t_m}, H_{< t_{m}},
\end{align}
for all $a_{\geq t_m}$ with  possible assigned treatments, where $\theta$ are the parameters of the causal model.
\end{theorem}

Thm. ~\ref{the1} is proved by \cite{bica2020time} and \cite{hatt2021sequential} in the discrete case.
Here, we extend Thm. ~\ref{the1} to the continuous-time setting. Nevertheless, 
there are still existing challenges in applying the deconfounder framework to longitudinal data in the continuous time setting.
After its original publication, \cite{wang2019blessings} has been met with concerns of difficulty in reconstructing confounders in practical applications and the deconfounder assumption itself has been challenged.
Towards the necessity of further constraints on the latent confounding, we introduce a frequency-based Lipschitz assumption on the structure of the hidden confounders in Assumption \ref{assump_freq_lip}.

\begin{assumption}
\label{assump_freq_lip}
\textnormal{Decomposition of time-varying hidden confounders.}
The hidden confounders $Z_t$ can be decomposed into  high-frequency components $Z^h_t$ and low-frequency $Z^l_t$ with distinguishable frequency gap $\omega$, i.e., $Z_{t_m} = (Z^h_{t_m}, Z^l_{t_m})$ such that low-frequency confounders have Lipschitz bounded influence and high-frequency confounders are sufficiently covered by proxy variables in $X_t$.
\end{assumption} 

In this sense, we combine two existing extensions of TSD under a unifying assumption.
$Z^l_{t}$ contains smooth information  (the trend of the confounding data) bounded by its maximal frequency $\omega_l$.
The  functional outcomes are then Lipschitz bounded by constant $L$.
Further, its distance and influence from its original value $Z_0$ will be bounded, reflecting its bounded variation from a static confounder $U$, as explored in \cite{hatt2021sequential}.
Further, the high-frequency components are assumed to have corresponding noisy proxy variables available in the measured covariates $X$.
Consequently, sufficient information about these high-frequency confounders can be derived from the observed proxy variables, as explored in \cite{kuzmanovic2021deconfounding}.
Unified together, our assumption explores a semiparametric assumption enhancing the practicality of applying the deconfounder setup to longitudinal data.

\begin{figure*}[t]
\centering
\includegraphics[width=\linewidth]{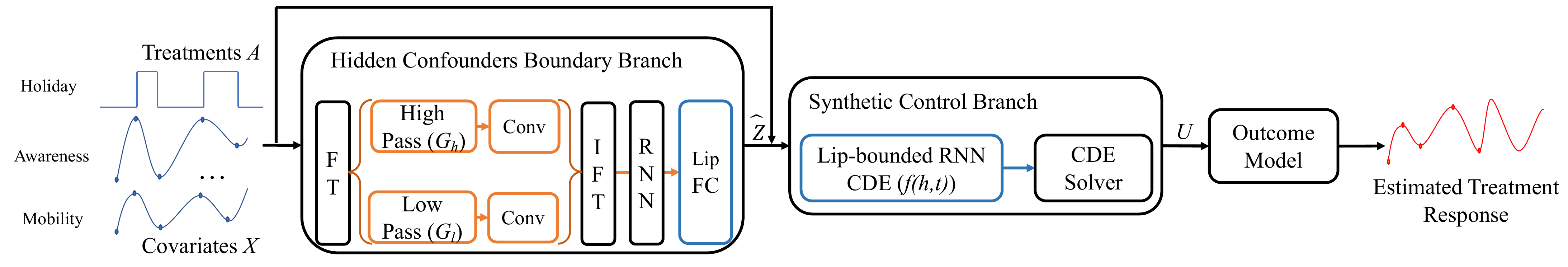}
\caption{Architecture of LipCDE.}
\label{fig:network}
\end{figure*}


\subsection{Neural Controlled Differential Equations}

 Starting from an initial state $u{(t_0)}$, neural ordinary differential equations (ODE) evolve following neural network based differential equations. The state at any time $t_i$ is given by integrating an ODE forward in time:
\begin{align}
    \frac{du(t)}{dt} = F(u(t),t;\theta), u(t_i) = u(t_0)+\int_{t_0}^{t_i} \frac{du(t)}{dt} dt,
\end{align}
where $F \in \mathcal{F}$, parametrized by $\theta$ with ($\mathcal{F}, || \cdot ||$) a normed vector space and $u{(t_0)}$ is the initial state. Neural CDEs are a family of continuous time models that explicitly define the latent vector field $f_{\theta}$ by a neural network parameterized by $\theta$, and allow for the dynamics to be modulated by the values of an auxiliary path over time. To constrain the ODE into  CDE format, let $\textbf{H}_t=(H^1_t,H^2_t,\cdots,H^n_t):t\in[t_0,t_m] \rightarrow \mathbb{R}^{n\times m}$ be the $m$ dimensional representation of historical  data with all $n$ observed history control paths, the integral be a Riemann-Stieltjes integral and $F$ be a continuous function acting on all control path~\cite{kidger2020neural}.  For continuous time synthetic control, we estimate the latent representation of treatment effect  $H_t$ through $H_t = H_{t_0} + \int_{t_0}^t f_\theta(H_s)d\textbf{H}_s, t\in (t_0, t_m].$

\section{Lipschitz Bounded Neural Controlled Differential Equations (LipCDE)}

To address the  treatment effect estimation task from irregular time series observation, we must avoid inference bias caused by hidden confounders. Thus, we propose an approach called Lipschitz bounded neural controlled differential equations (LipCDE). As shown in Figure~\ref{fig:network}, LipCDE  first infers the interrelationship of hidden confounders on treatment by bounding the boundary of hidden confounders via the hidden confounders boundary branch. After that,  LipCDE 
feeds the history trajectories into the synthetic control branch, which utilizes  both observed data and hidden confounders to  generate the latent representation of each patient. Besides, we re-weight the population of all participating patients and balance the representation via applying a time-varying  inverse probability of treatment weighting (IPTW) strategy. Combined with the LSTM layer, the outcome model can get the final estimate of the treatment effect.

\subsection{Hidden Confounders Boundary Branch }

In this section, we focus on how to use Lipschitz regularized convolutional operation to infer the hidden confounders from both high-frequency signals and low-frequency signals of observed data. 
As shown in Fig~\ref{fig:network}, the Fourier transform  $\mathcal{F}$ on observed data first converts the time-domain signals of history trajectories $h_t$ ~\cite{cao2020spectral, cao2021spectral}, including covariates and treatments with length $N$, into the corresponding amplitude and phase at different frequencies. Then, we sort the spectrum so that the spectrum corresponding to low-frequency information is concentrated at the origin after  Fourier transform, and high-frequency information is far from the origin and contains rich boundary and detail information.  
After that, we use Gaussian high-pass filter $G_{h}$ and Gaussian low-pass filter $G_{l}$ to get high-frequency components and low-frequency components, respectively: 

\begin{align}
\left\{\begin{aligned}
G_h(h_t) = G_{h}(\mathcal{F}(h_t)) =1-e^{\frac{-D^2(\mathcal{F}(h_t))}{2D^2_0}} \\
G_l(h_t) = G_{l}(\mathcal{F}(h_t)) =e^{\frac{-D^2(\mathcal{F}(h_t))}{2D^2_0}}
\end{aligned}\right.
\end{align}
The use of spectral-domain analysis enables change detection in certain frequency bands where the influence of trends (low frequency) or daily and seasonal cycles can be considered as time-invariant hidden confounders.
The high-frequency components are easily perturbed, which can be treated as noisy proxies. We extract the influence of hidden confounders on the covariates by analyzing the presence of the covariates we extract. After that, both components are fed  into convolutional operation: 
\begin{align}
    F_c(h_t) = Conv(G_h(h_t)) + Conv(G_l(h_t)) 
\end{align}

Next, we use the inverse Fourier transform $\mathcal{F}^{-1}$ converts the spectrum information of latent representation 
back to the time-domain signals. 
Then, the RNN layer takes the representation $\mathcal{F}^{-1}(F_c(h_t))$ as input and outputs the  hidden states $h_{hc}$ of hidden confounders.  Note that, after the Fourier transform, time series no longer consider specific timesteps in the spectral domain. In addition,  in contrast to directly handling irregular time series as ~\cite{ware1998fast},
we use the processing of the Fourier transform as a mathematical component without considering time intervals, and irregular sampling is enabled in the next component.

To define the boundary of hidden confounders' value interval,  following the RNN layer, the confounders encoder uses a  Lipschitz bounded linear fully-connected (FC) layer  with Lipschitz regularization~\cite{perugachidiaz2021invertible} to map the output of RNN layer into a hidden embedding, i.e. $z =g(h_{hc})=  W_{g}h_{hc} +b_g$.
The function $g : \mathbb{R}^n \rightarrow \mathbb{R}^K$ can be said as $L$-Lipschitz if there exists an $L$ such that for all $x,y \in \mathbb{R}^n$, we have $||f(x)-f(y)||\leq L||x-y||$~\cite{Bethune2021TheMF}. 
In this work, we enforce the function $g$ to satisfy the 1-Lipschitz constraint, where $g$ is  the linear FC layer. 
Following spectral normalization of \cite{gouk2021regularisation}:
\begin{align}
    \text{Lip}(g)\leq 1, \text{if}\hspace{.5em} ||W_{g}||_2 \leq 1,
\end{align}

where $||\cdot||_2$ is the spectral matrix norm, we enforce the linear weights $W_{g}$ to be at most 1-Lipschitz by having a spectral norm of less than one. This constraint ensures that when the observed data is within the normal interval, the inferred hidden confounders satisfy the corresponding bound interval with constant $L$.

\subsection{Synthetic Control Branch}

Since the neural ordinary differential equations(ODE) family is effective in continuous time problems, we use neural CDE to estimate latent factors and treatment effects. 
Inspired by~\cite{Bellot2021PolicyAU}, 
let $u_{t}:=g_{\eta}\left(H_{t}\right)= g_{\eta} \left([x_t,a_t,\hat{z}_t, H_{t-1}]\right)$, 
where $g_{\eta}: \mathbb{R}^{n\times m} \rightarrow \mathbb{R}^{l \times m}$ is a set of functions 
that embeds the historical data into a $l$-dimensional latent state.
 Let $f$ be a neural network parameterizing the latent vector field.
To apply Lipschitz constraint on $f$, following~\cite{erichson2020Lipschitz}, we define $f$ as a continuous time Lipschitz RNN:
\begin{align}
f(h,t)=A_R h+ \sigma \left(W_Rh+U u(s)+b\right),
\end{align}
where hidden-to-hidden matrices $A_{R}$ and $W_R$ are trainable matrices and nonlinearity $\sigma(\cdot)$ is an 1-Lipschitz function. 
Now $\dot{f}=\frac{\partial f(t)}{\partial t}$ is the time derivative and $f$  considers both controlling the history path of observed data and the hidden state of RNN.
A latent path can be expressed as the solution to a controlled differential equation of the form:
\begin{align}
\label{equ}
u_{t}=u_{t_{0}}+\int_{t_{0}}^{t} f\left(u_{s}, s\right) d \mathbf{H}_{s}^{0}, \quad t \in\left(t_{0}, t_{m}\right]
\end{align}

In that way, we can directly utilize adjoint methods ~\cite{chen2018neural} of CDEs to enable computing the gradient with a dynamic causal relationship between historical information controlled by $\mathbf{H}$ and outcomes. For each estimate of $f_{\theta}$ and $g_{\eta}$ the forward latent trajectory in time that these functions defined through (\ref{equ}) can be computed using any numerical ODE solver as those equations continuously incorporate incoming interventions, without interrupting the differential equation:
\begin{align}
  \hat{u}_{t_{1}}, \ldots, \hat{u}_{t_{k}}=\operatorname{ODESolve}\left(f_{\theta}, u_{t_{0}}, \mathbf{H}_{t_{1}}, \ldots, \mathbf{H}_{t_{k}}\right)  
\end{align}

\subsection{Outcome Model}
After sampling the latent representation $U_t = (\hat{u}_{t_{1}}, \ldots, \hat{u}_{t_{k}})$ of historical trajectories on each patient, we use the outcome model to estimate the treatment effect. 
To adjust the treatment assignment and get the final estimates,
we first re-weight the population via an RNN model, which can handle time-varying treatment assignment~\cite{lim2018forecasting}, to estimate the propensity scores and  IPTW of each dynamic time step.
After that, we use two stacked LSTM layers to decode the padded hidden sequence  of irregular inputs. Then we use a linear fully-connected layer mapping the output of the LSTM layer into an unbiased estimated treatment response over time. For the loss function part, we weight each patient via the generated score of IPTW, $w^i$, and use the  mean squared error (MSE) function as our target loss function: $L =\frac{1}{N}\sum_{i=1}^{N}w^i(\hat{y}^i_{t_{m+1}} - y^i_{t_{m+1}})^2$.

Empirically, the identifiability can be assessed on the synthetic data via sample hidden confounders $Z_t$ repeatedly to evaluate the uncertainty of the outcome model estimates. However, identifiability might not be guaranteed under the framework of deconfounding in the completely general case \cite{DAmour2019OnMC,ogburn2020counterexamplesToDeconfounder}. Previous works find that the estimates may have a high variance when the treatment effects are non-identifiable ~\cite{bica2020time, hatt2021sequential, kuzmanovic2021deconfounding}. To achieve the goal of identifiability and obtain unbiased ITE estimates, ~\cite{hatt2021sequential} introduces the assumption of \textit{Time-Invariant Unobserved Confounding}, which requires the hidden confounders are invariant for different timestamps, and ~\cite{kuzmanovic2021deconfounding} claim that we can learn the hidden embedding to make \textit{Sequential Strong Ignorability} assumption hold via the observed noised proxies. Thus, the  greater identifiability of our work follows both ~\cite{hatt2021sequential} and ~\cite{kuzmanovic2021deconfounding} as it utilizes both time-invariant hidden confounders from low-frequency components  and dynamic noisy proxies from the high-frequency component of the observed data  simultaneously in practice.

\begin{figure*}[t]
\centering

\subfigure[RMSE results on treatment effects]{
\begin{minipage}[t]{0.5\linewidth}
\centering
\includegraphics[height=4.5cm]{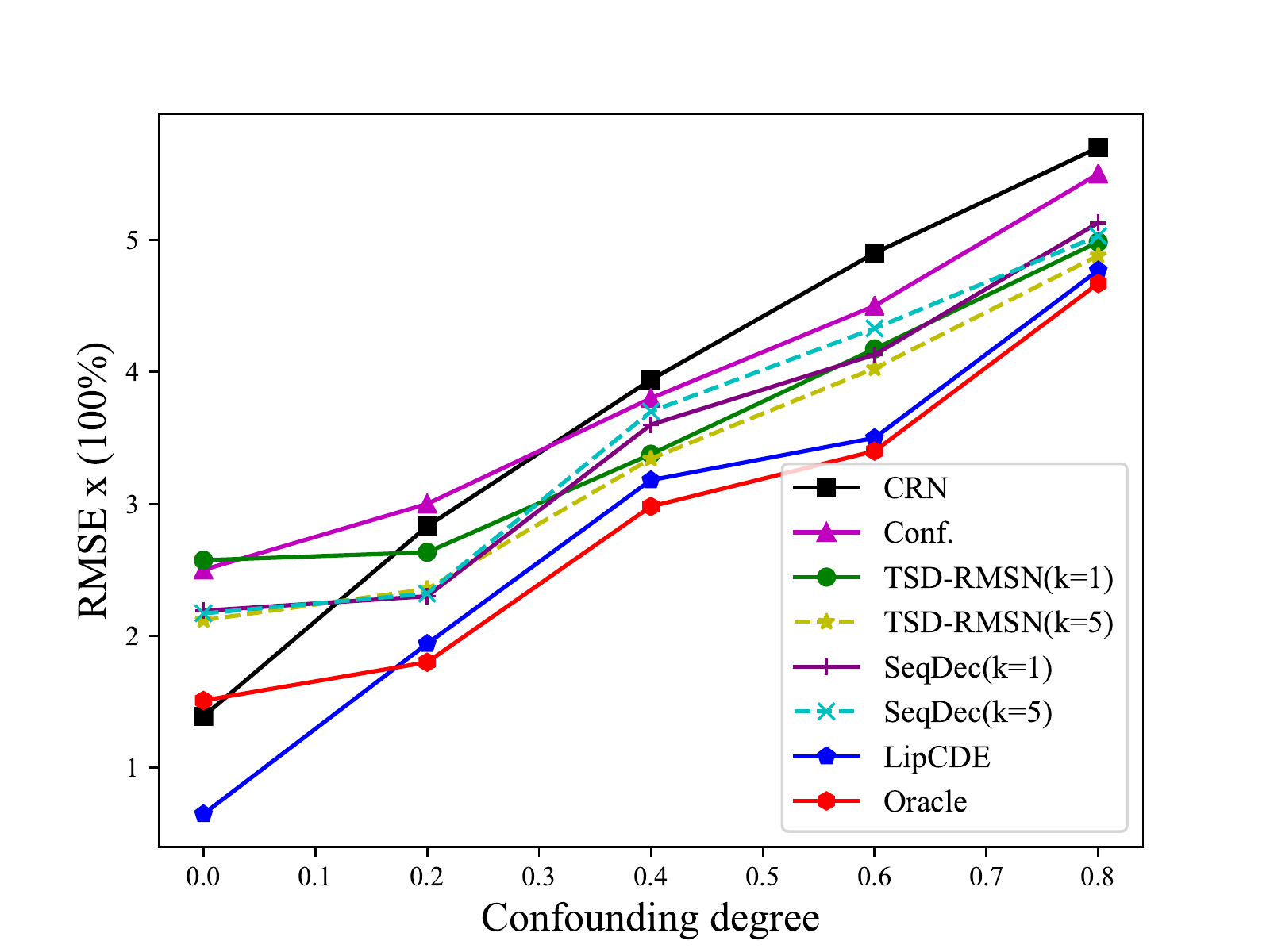}
\label{analysis_fin1}
\end{minipage}%
}%
\subfigure[RMSE results on counterfactual treatment effects]{
\begin{minipage}[t]{0.5\linewidth}
\centering
\includegraphics[height=4.5cm]{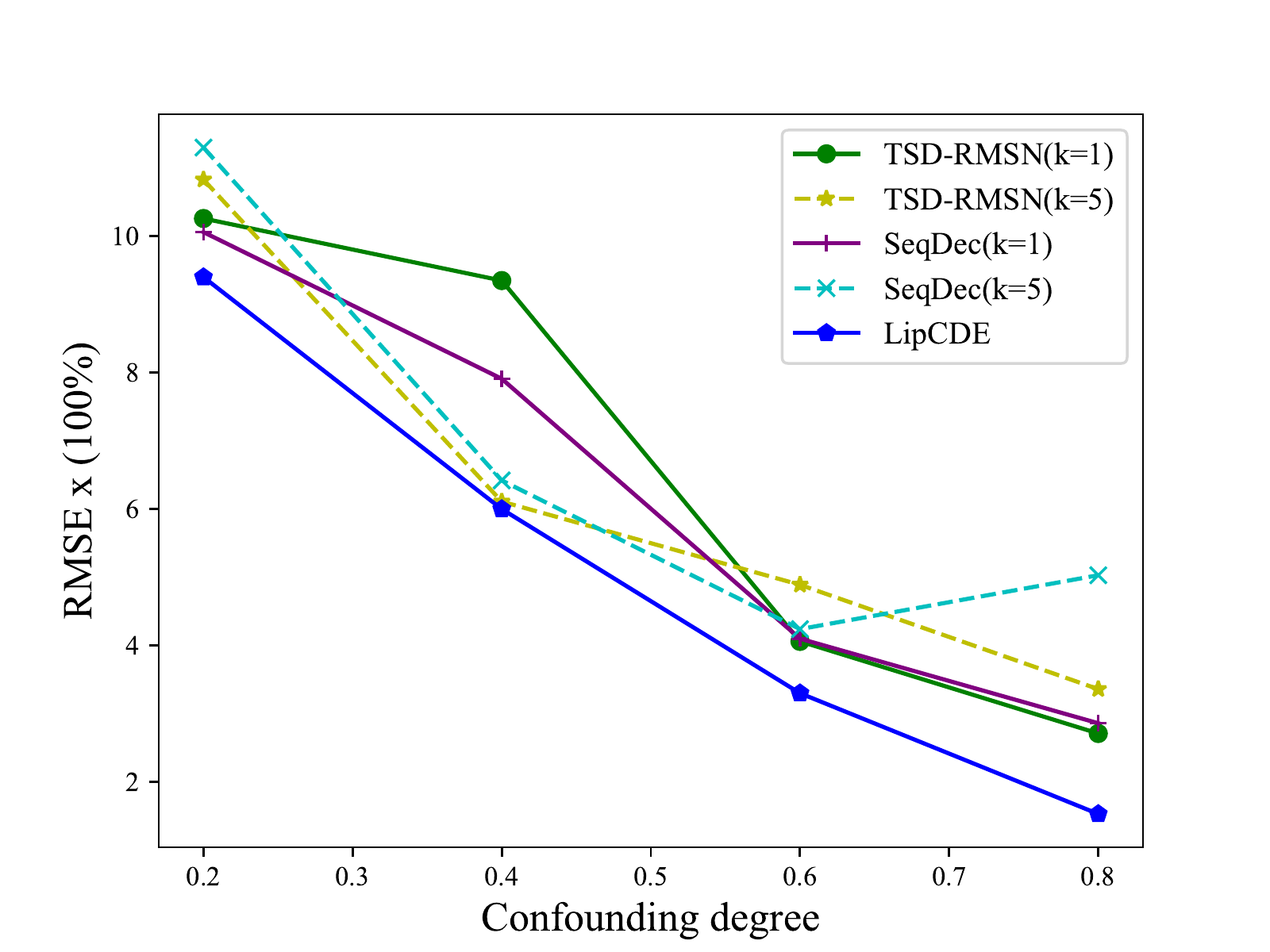}
\label{analysis_fin2}
\end{minipage}%
}

\caption{Results on synthetic data}
\label{analysis_fin3}
\end{figure*}
\section{Experiments}

\subsection{Experiments Setting}

In this section, we estimate the treatment effects for each time step by one-step ahead predictions on both synthetic dataset and real-world datasets including MIMIC-III~\cite{johnson2016mimic} dataset and COVID-19~\cite{steiger2020causal} dataset. Hidden confounders in such real-world datasets is present as variables not included in the records. However, for real-world data, it is untestable to estimate the oracle treatment responses and we only evaluate the factual treatment effects.

\textbf{Baselines.} LipCDE is evaluated by examining the degree of control it has over hidden confounders.
The baselines used in these experiments are: \textbf{Oracle}, which estimates ITE with simulated (oracle) confounders; \textbf{Conf. (No-hidden)}, which assumes no hidden confounders and can make it clear how hidden confounders here impact the performance of treatment effect prediction models; \textbf{CRN}~\cite{Bica2020Estimating}, which introduces a sequence-to-sequence counterfactual recurrent network to estimate treatment effects and utilizes domain adversarial training to  handle the bias from time-varying confounders;
\textbf{TSD}~\cite{bica2020time}, which leverages the assignment of multiple treatments over time to enable the estimation of treatment effects in the presence of multi-cause hidden confounders; \textbf{DTA}~\cite{kuzmanovic2021deconfounding}, which combines a 
LSTM autoencoder with a causal regularization penalty to learn dynamic noisy proxies and render the potential outcomes and treatment assignment conditionally independent; 
\textbf{SeqDec}~\cite{hatt2021sequential}
, which utilizes a Gaussian process latent variable model to infer substitutes for the hidden confounders; \textbf{OriCDE}~\cite{Bellot2021PolicyAU}, which can estimate ITE explicitly using the formalism of linear controlled differential equations. 

\textbf{Outcome model. } Except OriCDE, all baselines share the same design of the outcome model, i.e. \textit{MSM}~\cite{robins2000marginal}, which uses inverse probability of treatment weighting (IPTW) to adjust for the time-dependent confounding bias by linear regression and then constructs a pseudo-population to compute final outcome, and \textit{RMSN}~\cite{lim2018forecasting}, which estimates IPTW using RNNs instead of logistic regressions. OriCDE and LipCDE use the outcome model introduced in previous section. 

\begin{table*}[htpb]
\centering
\small
\begin{tabular}{ccccccccccc}
\hline\hline
Outcome Model & -& \multicolumn{3}{c}{MSM (RMSE\%)}& \multicolumn{4}{c}{RMSN (RMSE\%)}& \multicolumn{2}{c}{Ours (RMSE\%)}                 \\ \cline{2-11} 
Methods       & \multicolumn{1}{c|}{CRN}&Conf.  & DTA   & \multicolumn{1}{c|}{TSD}   & Conf.  & DTA   & TSD   & \multicolumn{1}{c|}{SeqDec} & OriCDE & LipCDE                          \\ \hline
 Blo. pre.     & \multicolumn{1}{c|}{12.43}&14.54 & 13.31 & \multicolumn{1}{c|}{13.57} & 14.46 & 18.33 & 12.11 & \multicolumn{1}{c|}{13.74}  & 10.55  & \textbf{9.19}  \\ 
 Oxy. sat.     &\multicolumn{1}{c|}{4.17}& 4.72  & 4.65  & \multicolumn{1}{c|}{4.33}  & 4.22  & 4.21  & 4.25  & \multicolumn{1}{c|}{4.19}   & 4.24   & \textbf{4.15 } \\ 
 COVID-19      & \multicolumn{1}{c|}{-}&15.10 & 13.93 & \multicolumn{1}{c|}{13.07} & 11.48 & 13.52 & 11.08 & \multicolumn{1}{c|}{11.43}  & 11.36  & \textbf{7.56}  \\ \hline \hline
\end{tabular}

\caption{Results for real-world data (MIMIC-III and COVID-19) experiments. Lower is better.}
\label{tab:mimic_real}
\end{table*}
\begin{table*}[t]
\centering
\small
\begin{tabular}{c|c|cccc|c|cccc}

\hline \hline 
Degree &MR                                & Conf. & TSD  & SeqDec & LipCDE & MR                                & Conf. & TSD  & SeqDec & LipCDE\\ \hline 
0& \multirow{3}{*}{15\%}                          & 3.43   & 2.83 & 2.43   & \textbf{1.19} & \multirow{3}{*}{30\%}  &  3.32   & 2.84 & 3.19   & \textbf{2.29}  \\
                                   0.2    &       &  3.47   & 2.84 & 2.69   & \textbf{2.6} &&4.66   & 3.65 & 2.95   & \textbf{2.62}    \\
                                0.4    &           & 3.45   & 3.67 & 3.7    & \textbf{3.39}  && 4.19   & 4.06 & 3.89   & \textbf{3.61}   \\ 
                                 \hline \hline
\end{tabular}
\caption{Irregular data with missing value}
\label{tab:irr}
\end{table*}

\subsection{Estimating Treatment Effects Experiments }

\textbf{Synthetic experiments.}
For the synthetic dataset, we can simulate data in which we are able to control the amount of hidden confounders and decide the treatment plan. Therefore in addition to estimating factual treatment responses, we will also perturb the inputs to quantify how accurate counterfactual relationships are captured by LipCDE.
Following~\cite{bica2020time}, we have $T=30$ max time steps and $N=5000$ patient trajectories, where each patient has $p =5$ observed covariates and different treatments.  We vary the confounding degree parameter $\gamma$ to produce a varying amount of hidden confounders.  
Factual results use the outcome results corresponding to the real-world treatment we simulate. For the counterfactual estimations, we set all the treatments to 0 at the timestamp interval of $[\frac{l_i}{2},l_i]$, where $l_i$ is the sequential length of patient $i$, and get the outcome of the counterfactual world. 
As shown on Figure~\ref{analysis_fin3}, for the factual treatment effects results, methods considering hidden confounders are  generally better than the models without the hidden confounders (CRN, Conf.). Note that, LipCDE achieves better results on all different levels of confounders and its outcome is closest to the estimates obtained using simulated (oracle) confounders, which means LipCDE can yield less biased estimates compared with other baselines. In addition, LipCDE remains stable and becomes closer to the simulated (oracle) confounders baseline when we increase the degree of confounders' influence, which indicates that our model can effectively constrain the influence boundary of hidden confounders based on observed data. For the counterfactual path results, we interestingly observe that the RMSE decreases as the confounding degree increases. The reason is that  when the degree increase, $Z_t$ gets easier to handle with fixed treatment plans referring to the data generation method. Besides, LipCDE still performs better than the current baselines in the counterfactual world, indicating the stability of LipCDE for hidden confounder reasoning and the validity of the estimation.

\textbf{Real-world experiments on MIMIC-III \& COVID-19.}  
real-world data allow us to demonstrate LipCDE has strong scalability and interpretability in real-world applications. 
MIMIC-III dataset
contains 5000 patient records with 3 treatments, 20 covariates of patients and 2 outcomes including blood pressure (Blo. pre.), and oxygen saturation (Oxy. sat.). 
The COVID-19 dataset contains 401 German districts over the period of 15 February to 8 July 2020. We extract 10 time-varying covariates and 2 treatments with 2 outcomes, 'active cases', in each district. 
The results in Table \ref{tab:mimic_real} show that LipCDE outperforms existing baselines in all cases. By modeling the dependence of the assigned treatment for each patient, LipCDE is able to infer latent variables and make orderly use of the causal relationship between latent variables and observed data. This result is consistent with what we have seen in the simulated dataset.
Specifically, the average  RMSE on MIMIC-III's blood pressure outcome and COVID-19 datasets is decreased by 28.7\%  and 32.3\% over TSD and SeqConf respectively. 
Besides, the small increase in oxygen saturation is thought to be due to the fact that oxygen saturation itself is not dependent on current covariates and is less influenced by treatment.
Although these results on real data require further validation by physicians, they demonstrate the potential of the method to be applied in real medical scenarios.

\begin{figure}[h]
\centering
\includegraphics[width=\linewidth]{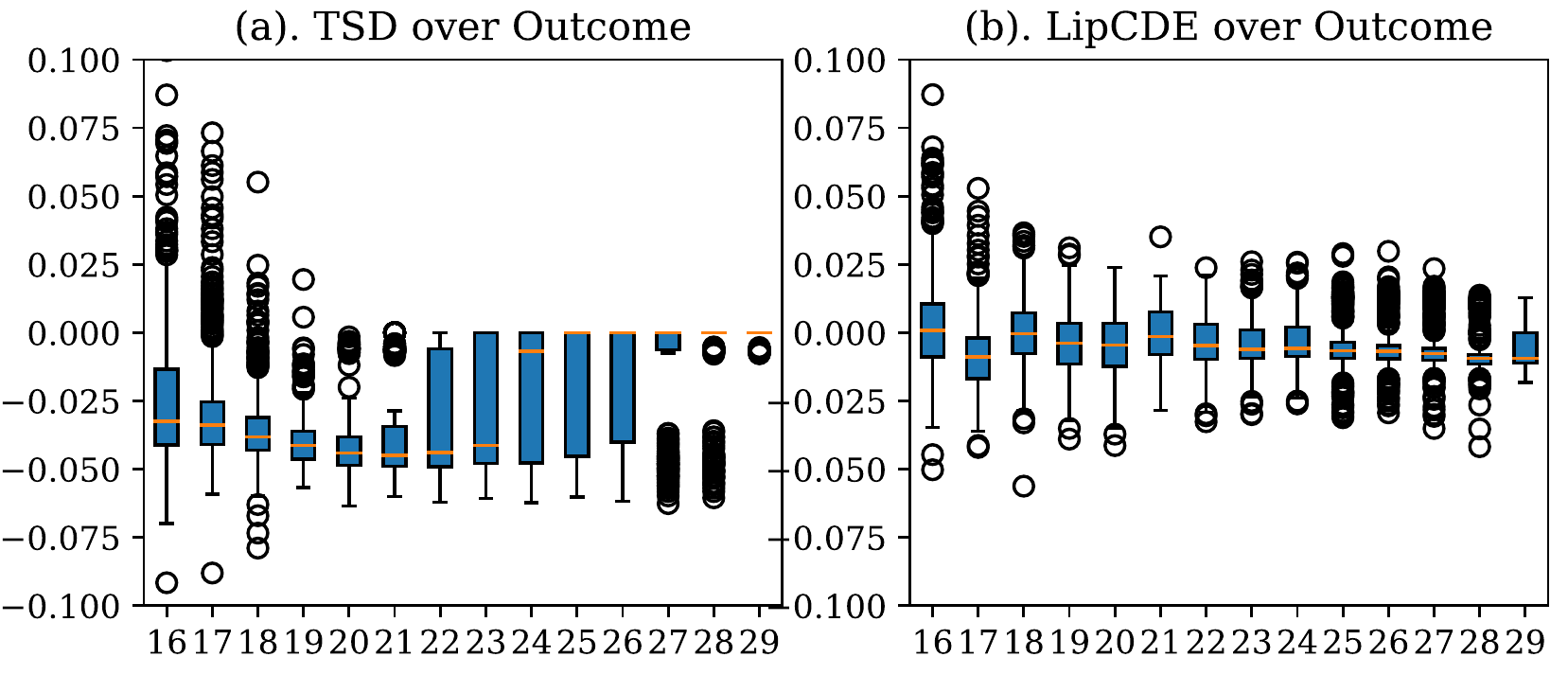}
\caption{Analysis of the outcome  on synthetic data's counterfactual path. Compared with baseline models, LipCDE can estimate treatment effects with lower variance}
\label{fig:box_0.8}
\end{figure}

\section{Analysis}

\textbf{Irregular time series with missing values.} We emphasize that our model is suitable for irregular time series sampling. Therefore, we randomly remove 15\% and 30\% of the aligned synthetic data with different confounding degrees, independently for each unit. Except for CDE-based methods, all the baselines require some form of prior interpolation. 
Results shown in Table~\ref{tab:irr} demonstrate that our model achieves a comparable  performance with irregularly aligned data. 
Note that, compared with SeqDec which only models irregular samples via an indirect simple multivariate Gaussian distribution, LipCDE shows the ability to handle continuous time settings by utilizing the CDE module.

\textbf{Analysis on bounded hidden confounders.} 
We perform the analysis using simulated datasets and evaluate the hidden confounder's quality on  LipCDE with TSD  and SeqDec. As shown on Figure~\ref{fig:box_0.8}, LipCDE can achieve better estimate results with lower variance compared with the previous strong baseline. 
Further, we find that TSD can induce highly confident posterior distributions with lower bounds of the hidden confounders, which can yield highly confident biased predictions~\cite{zheng2021bayesian}. 
The seqDec model has more discrete points and no obvious boundary, which also leads to the degradation of the model performance. LipCDE controls the data distribution of hidden confounders more accurately by filtering the convolutional neural network and Lipschitz regularization, which has  higher similarity to the originally hidden confounder compared with other baselines.

\section{Conclusion}

In this paper, we proposed the Lipschitz-bounded neural controlled differential equation (LipCDE), a novel neural network that utilizes hidden confounders for estimating treatment effect in the case of irregular time series observations. For one thing, it uses the performance of time-varying observations in the frequency domain to infer the hidden confounders under Lipschitz regularization.  For another thing, a well-designed CDE explicitly models  the combinational latent path of observed time series, which can effectively capture underlying temporal dynamics and intervention effects. 
With experimental results on synthetic and real datasets, we demonstrate the effectiveness of LipCDE in reducing bias in the task of estimating  treatment effects.

\subsubsection{Acknowledgements}  This project is partially supported by KDDI Research, Inc.  Defu Cao, James Enouen, and Chuizheng Meng are partially supported by 
the Annenberg Fellowship of the University of Southern California.  

\newpage
\nocite{langley00}

\bibliography{aaai23}

\appendix
\onecolumn


\section{ Details for Experiments}
\label{app_exp_detail}
\subsection{Experiments Setting}
LipCDE described in Experiments Section  was implemented in PyTroch and trained on 4 NVIDIA GeForce RTX 2080 TI GPUs. We adopt the Adam~\cite{kingma2014adam} optimizer with learning rate 0.01. The training epoch is set as 10 with 10 iterations on each batch. The batch size is 16 and each dataset follows a 80\%/10\%/10\% split for training/validation/testing respectively.  LipCDE is trained by an end-to-end way with the same loss function with ~\cite{Bellot2021PolicyAU}. In addition, we report the average RMSE over 10 model runs on each experiment and report the mean results of LipCDE.

Specifically, in hidden confounders boundary branch, we apply RNN on that branch with 32 hidden units which are decided by grid search from [16, 32, 64, 128] on the validation dataset and then apply Lipschitz constraint on FC layer without tunable hyperparameters. For the synthetic control branch, we adapt Lipschitz bounded RNN from ~\cite{erichson2020lipschitz2} with:
\begin{align}
\left\{\begin{aligned}
A_{\beta_{A}, \gamma_{A}} &=\left(1-\beta_{A}\right)\left(M_{A}+M_{A}^{T}\right)+\beta_{A}\left(M_{A}-M_{A}^{T}\right)-\gamma_{A} I \\
W_{\beta_{W}, \gamma_{W}} &=\left(1-\beta_{W}\right)\left(M_{W}+M_{W}^{T}\right)+\beta_{W}\left(M_{W}-M_{W}^{T}\right)-\gamma_{W} I
\end{aligned}\right.
\end{align}
where $\beta_{A}, \beta_{W} \in[0,1], \gamma_{A}, \gamma_{W}>0$ are tunable parameters and hidden-to-hidden matrices $A_{\beta, \gamma} \in \mathbb{R}^{N \times N}$ and $W_{\beta, \gamma} \in \mathbb{R}^{N \times N}$ are trainable matrices. The LSTM layers in our outcome model are with 64,32 hidden states, respectively, which are also decided by grid search.

\subsection{Simulated Dataset }
\label{app_exp_sim}

Synthetic data allows us to simulate data in which we can control the amount of hidden confounders.  
Following~\cite{bica2020time}, the observed data of patients is $H = (\{x_t^i,a_t^i,y_{t+1}^i\}^T_{t=1})^N_{i=1}$ and the hidden confounders is $(\{z_t^{i}\}^T_{t=1})^N_{i=1}$, where we have $T=30$ max time steps and $N=5000$ patient trajectories, where each patient has $p =5$ observed covariates and different treatments.  We vary the confounding degree parameter $\gamma \in \{0, 0.2, 0.4, 0.6, 0.8\} $ to produce a varying amount of hidden confounders.   
Factual results use the outcome results corresponding to the real-world treatment we simulate. For the counterfactual path, we set all the treatments to 0 at the timestamp interval of $[\frac{l_i}{2},l_i]$, where $l_i$ is the sequential length of patient $i$, and get the outcome of the counterfactual world. Then, we calculate the RMSE of model simulated outcome and counterfactual world outcome.

We build a dataset using $p$-order autoregressive processes. At each timestep $t$, we simulate $k$ time-varying covariates $X_{t, k}$ representing single cause confounders and a multi-cause hidden confounders $Z_{t}$ as follows:
\begin{align}
X_{t, j} &=\frac{1}{p} \sum_{i=1}^{p}\left(\alpha_{i, j} X_{t-i, j}+\omega_{i, j} A_{t-i, j}\right)+\eta_{t} 
\end{align}
\begin{align}
Z_{t} &=\frac{1}{p} \sum_{i=1}^{p}\left(\beta_{i} Z_{t-i}+\sum_{j=1}^{k} \lambda_{i, j} A_{t-i, j}\right)+\epsilon_{t}
\end{align}

for $j=1, \ldots, k, \alpha_{i, k}, \lambda_{i, j} \sim \mathcal{N}\left(0,0.5^{2}\right), \omega_{i, k}, \beta_{i} \sim$ $\mathcal{N}\left(1-(i / p),(1 / p)^{2}\right)$, and $\eta_{t}, \epsilon_{t} \sim \mathcal{N}\left(0,0.01^{2}\right)$. The value of $Z_{t}$ changes over time and is affected by the treatment
assignments.

Each treatment assignment $A_{t, j}$ depends on the single-cause confounders $X_{t, j}$ and multi-cause hidden confounders $Z_{t}$ :
\begin{align}
\pi_{t j} &=\gamma_{A} \hat{Z}_{t}+\left(1-\gamma_{A}\right) \hat{X}_{t j} \\
A_{t j} \mid \pi_{t j} & \sim \operatorname{Bernoulli}\left(\sigma\left(\lambda \pi_{t j}\right)\right)
\end{align}
where $\hat{X}_{t j}$ and $\hat{Z}_{t}$ are the sum of the covariates and confounders respectively over the last $p$ timestamps, $\lambda=15$, $\sigma(\cdot)$ is the sigmoid function and $\gamma_{A}$ controls the amount of hidden confounding applied to the treatment assignments. The outcomes are also obtained as a function of covariates and hidden confounders.
\begin{align}
\mathbf{Y}_{t+1}=\gamma_{Y} Z_{t+1}+\left(1-\gamma_{Y}\right)\left(\frac{1}{k} \sum_{j=1}^{k} X_{t+1, j}\right),
\end{align}
where $\gamma_{Y}$ controls the amount of hidden confounding applied to the outcome. We simulate datasets consisting of 5000 patients, with trajectories between 20 and 30 timestamps, and $k=3$ covariates and treatments. To induce time dependencies we set $p=5$.

\subsection{real-world dataset}
~\label{Realdata}

MIMIC-III dataset
contains three treatment options: antibiotics, vasopressors, and mechanical ventilators. We extract 20 covariates of patients, including laboratory tests and vital signs for each patient and 2 outcomes including blood pressure (Blo. pre.), and oxygen saturation (Oxy. sat.). We extract up to 30 days of 5000 patient records from the dataset for training and testing following the same setting with \cite{bica2020time}, and infer treatment response within 24 hours. 

Then, we apply experiments on  the COVID-19 dataset, which contains 401 German districts over the period of 15 February to 8 July 2020.  We extract 10 time-varying covariates which focus on mobility including parks mobility,  workplaces mobility, etc; weather including rainfall and  temperature; awareness such as searches corona, etc. The task is to infer the effects of multiple treatments including 'Holiday' and 'Weekday' over the 10 covariates for the outcome 'activate cases' in each district.%

\section{Additional Analysis}
\label{app_analysis}
In general, models that do not take into account hidden confounders cannot obtain correct causality because they assume that treatment assignment only depends on the observed history, which means that any unobserved probability confounder can lead to a biased estimate of outcome. LipCDE obtains the final representation used to infer the outcome by analysing all possible factors, thus reducing the bias caused by the presence of hidden confounders. In this section, we show additional LipCDE analysis for Analysis Section.

\begin{table*}[htbp]
\centering
\begin{tabular}{cccccc}
\hline\hline
Miss  Rate                                   & Degree & Conf. & TSD-RMSN(K=5)  & SeqDec(K=5) & LipCDE \\ \hline 
\multirow{5}{*}{0\%}                     & 0      & 2.5   & 2.11 & 2.17   & \textbf{0.65}   \\
                                          & 0.2    & 3.01   & 2.633 & 2.32   & \textbf{1.94}    \\
                                          & 0.4    & 3.83   & 3.37 & 3.70    & \textbf{3.18}   \\ 
                                          & 0.6    & 4.51   & 4.17 & 4.33   & \textbf{3.50}   \\
                                            & 0.8    & 5.55   & 4.98 & 5.03    & \textbf{4.77}   \\\hline \hline 
\multirow{5}{*}{15\%}                     & 0      & 3.43   & 2.83 & 2.43   & \textbf{1.19}   \\
                                          & 0.2    & 3.47   & 2.84 & 2.69   & \textbf{2.60}    \\
                                          & 0.4    & 3.45   & 3.67 & 3.70    & \textbf{3.39}   \\
                                          & 0.6    & 6.28  & 4.94 & 5.28    & \textbf{4.45}   \\
                                            & 0.8    & 7.07  & \textbf{5.48} & 6.17    & 5.62   \\
                                            \hline \hline 
\multirow{5}{*}{30\%} & 0      & 3.32   & 2.84 & 3.19   & \textbf{2.29}   \\
\multicolumn{1}{c}{}                      & 0.2    & 4.66   & 3.65 & 2.95   & \textbf{2.62}  \\
\multicolumn{1}{c}{}                      & 0.4    & 4.19   & 4.06 & 3.89   & \textbf{3.61}  
\\
 & 0.6    & 6.90   & 6.87 & 6.33   & \textbf{4.66}   \\
  & 0.8    & 10.06  & 7.70 & 8.01  & \textbf{5.91}    \\\hline\hline
   
\end{tabular}
\caption{Irregular data with missing data rate in \{0\%, 15\%, 30\%\}.}
\label{tab:irr_app}
\end{table*}

\begin{table}[t]
\setlength{\abovecaptionskip}{0.5cm} 
\centering
\begin{tabular}{ccccc}
\hline \hline Outcome  &\multirow{2}{*}{Method} &\multicolumn{3}{c}{ RMSE (\%) } \\
\cline { 3 - 5} model&& Blo. pre. & Oxy. sat. & COVID-19 \\
\hline 
\multirow{5}{*}{Ours} 
&LipCDE & $\textbf{9.19}\pm 0.36 $ & $\textbf{4.15} \pm 0.37$& $\textbf{7.56} \pm 0.59$ \\
&-\textit{w/o-hc} & $10.54 $ & $4.27$ & $10.69$  \\
&-\textit{w/o-lip} & $9.51 $ & $4.24$ & $10.5$  \\
&-\textit{w/o-high} & $10.10 $ & $4.21$ & $10.69$  \\
&-\textit{w/o-low} & $9.73 $ & $4.29$ & $10.5$  \\
\hline \hline
\end{tabular}
\caption{Ablation study for real-world data (MIMIC-III and COVID-19) experiments. Lower is better.}
\label{tab:mimic_app}
\end{table}

\subsection{Irregular time series with missing values}
For \textbf{irregular time series with missing values} analysis part, we  add the results when there is no data missing. It can be seen  from Table~\ref{tab:irr_app} that our model outperforms the baseline in all cases. 

\subsection{Ablation study}
For \textbf{ablation study} part,  we conduct ablation experiments to verify the effectiveness of the proposed components on the real-world datasets, which can clearly prove that our proposed method is an effective model for estimating treatment effect. Here, we discuss 4 variants of LipCDE: \textit{w/o-hc}, which estimates ITE without considering hidden confounders in our model structure; \textit{w/o-lip}, which estimates ITE without Lipschitz constraint; \textit{w/o-high} and \textit{w/o-low}, which reduce the high-frequency components  and low-frequency component in LipCDE, respectively. As shown in Table~\ref{tab:mimic_app}, the estimation error of \textit{w/o-hc} is larger than that with hidden confounders, demonstrating that our proposed method can take advantage of the hidden information to estimate the treatment effects better. Besides, the gap between \textit{w/o-lip} and LipCDE shows that the Lipschitz regularization effectively avoids the negative impact of the presumed anomaly hidden confounders on the outcomes. Furthermore, we examine that both high-frequency information and low-frequency information can contribute to reducing the variance of estimating treatment effects, which is aligned with the separate statement on dynamic noisy proxies paper~\cite{kuzmanovic2021deconfounding} and time static confounders paper~\cite{hatt2021sequential}. In addition, we report the average RMSE over 10 model runs on each experiment and report the mean and standard deviation of LipCDE.  

\subsection{Analysis on  hidden confounders}
For \textbf{analysis on bounded hidden confounders}  part, we further use "covariance similarity score (CovSim)"\cite{wu2020understanding, fan2022depts} to show the similarity between true confounders and hideen confounders inferred by proposed methods. 
Given true confounders $i$ and  hideen confounders $j$, the “covariance similarity score" between them can be calculated by:
\begin{align}
\text{CovSim}_{i,j} = \frac{||(U_{i,r_i}D_{i,r_i}^{1/2})^\top U_{j,r_j} D_{j,r_j}^{1/2}||_F}{||U_{i,r_j}D_{i,r_i}^{1/2}||_F \cdot||U_{j,r_j}D_{i,r_j}^{1/2}||_F}, 
\end{align}
where $U_{i,r_i}$ and $D_{i,r_i}$ denote matrices consisting of the top $r_i$ eigenvalues and eigenvectors respectively, calculated by SVD over the instance-level representation matrix for each task $i$; 
$r_i$ is chosen to contain 99\% of the eigenvalues. As a result, the mean CovSim on different confounding degrees from 0.0 to 0.8 between true confounders and hidden confounders inferred by LipCDE is 0.82, which is \textbf{15\%}  and \textbf{30\% }higher than  the CovSim between true confounders and hidden confounders inferred by TSD-RMSN (0.71) and by SeqConf (0.63), respectively.  This shows that the hidden confounders extracted by LipCDE matches the real information to a greater extent than other baselines, proving that our proposed model can effectively constrain the hidden confounders.

For \textbf{analysis on outcomes of counterfactual path}  part, we plot all the degrees' results on Figure\ref{fig:box_app_0.0}, \ref{fig:box_app_0.2}, \ref{fig:box_app_0.4}, \ref{fig:box_app_0.6}, \ref{fig:box_app_0.8}. We plot the difference between the estimated outcome and the synthetic outcome on the counterfactual path. In addition to using the RMSE to illustrate our model's performance directly, when estimating the treatment outcome, LipCDE is able to make more accurate estimates above different time steps.

\begin{figure*}[h]
\centering
\includegraphics[width=\linewidth]{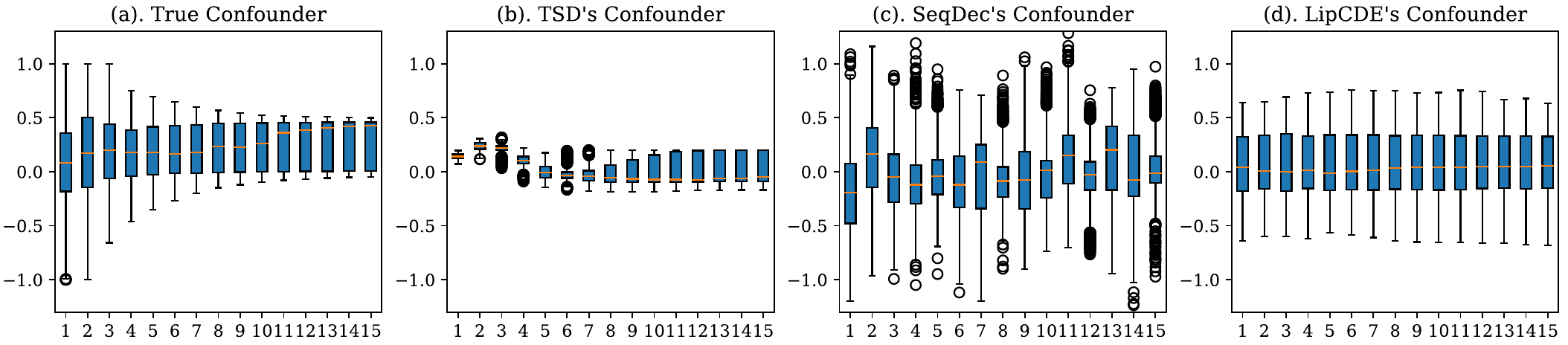}
\caption{Analysis of the hidden confounders' boundary  on synthetic data with first 15 timestamps. Here, the closer the shape of the box plot is to the true confounder, the less discrete the value is, the more accurate we consider the hidden confounder.}
\label{fig:box}
\end{figure*}

\begin{figure*}[htbp]
\centering
\includegraphics[width=\linewidth]{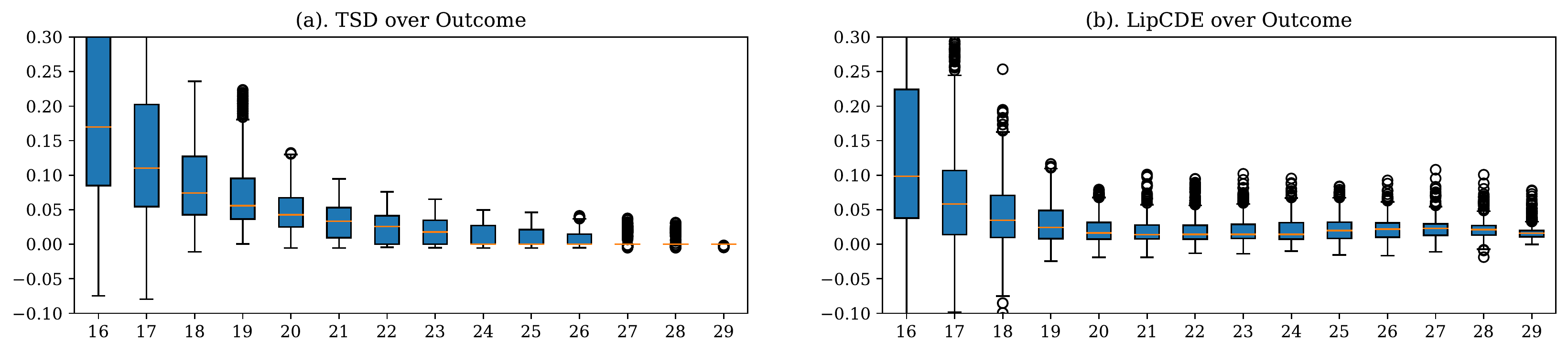}
\caption{Analysis of the outcome  on synthetic data's counterfactual path with degree 0.0.}
\label{fig:box_app_0.0}
\end{figure*}

\begin{figure*}[htbp]
\centering
\includegraphics[width=\linewidth]{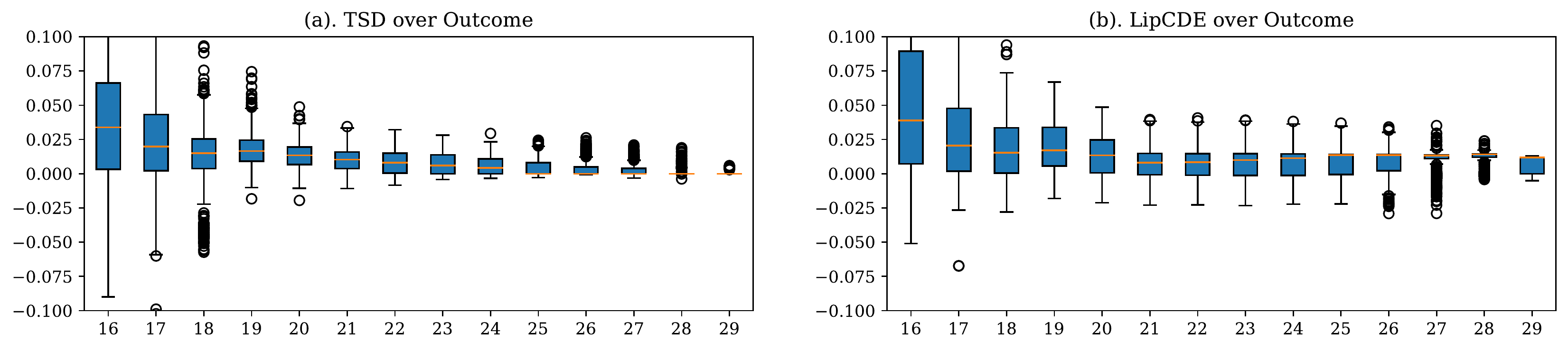}
\caption{Analysis of the outcome  on synthetic data's counterfactual path with degree 0.2.}
\label{fig:box_app_0.2}
\end{figure*}

\begin{figure*}[htbp]
\centering
\includegraphics[width=\linewidth]{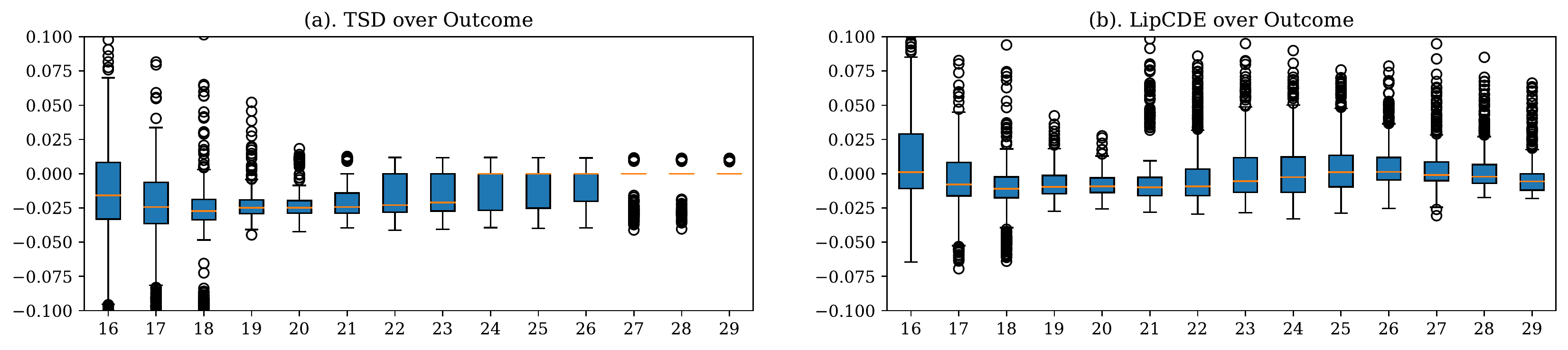}
\caption{Analysis of the outcome  on synthetic data's counterfactual path with degree 0.4.}
\label{fig:box_app_0.4}
\end{figure*}

\begin{figure*}[htbp]
\centering
\includegraphics[width=\linewidth]{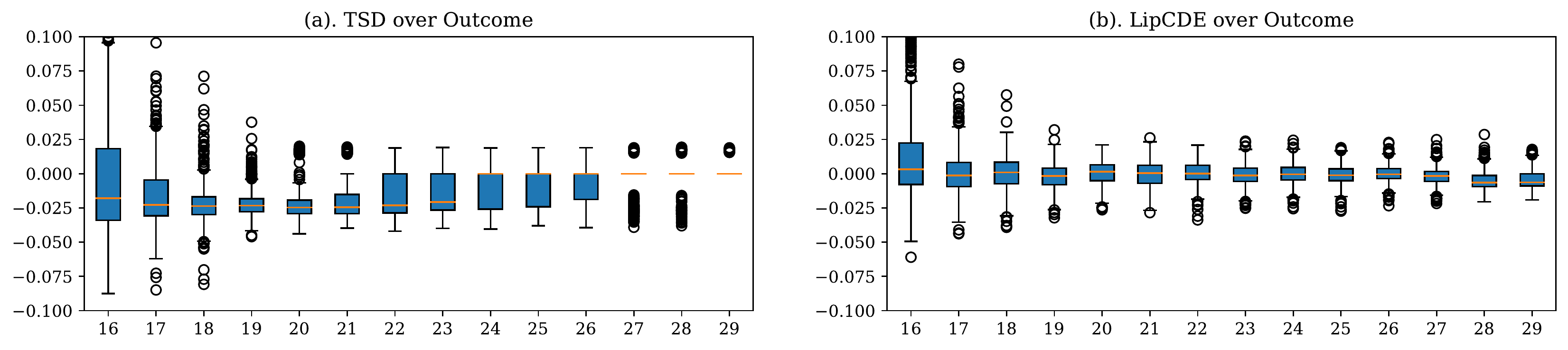}
\caption{Analysis of the outcome  on synthetic data's counterfactual path with degree 0.6.}
\label{fig:box_app_0.6}
\end{figure*}

\begin{figure*}[htbp]
\centering
\includegraphics[width=\linewidth]{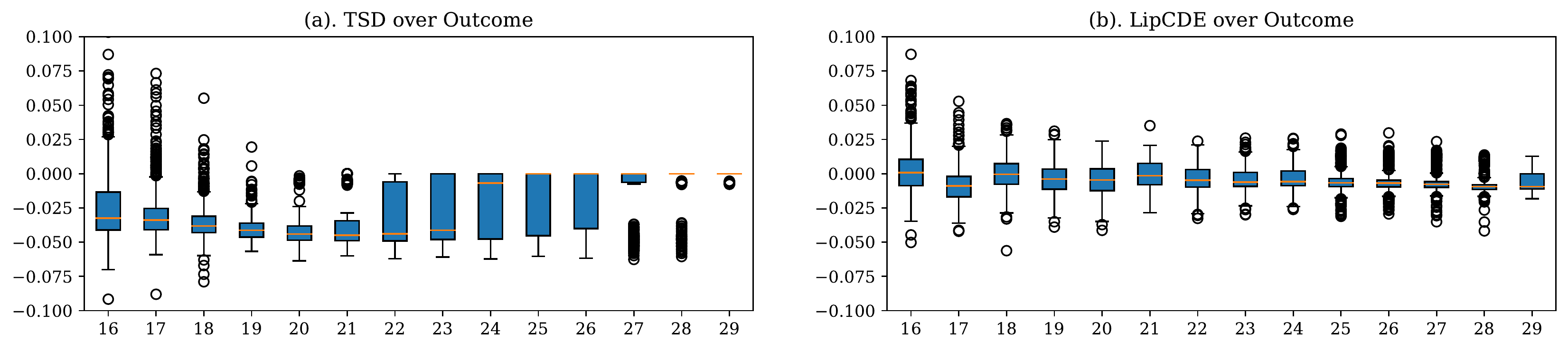}
\caption{Analysis of the outcome  on synthetic data's counterfactual path with degree 0.8.}
\label{fig:box_app_0.8}
\end{figure*}

\section{Theoretical Discussion}
\label{app_deconfounding_theory}

\subsection{The Deconfounding Assumption}

\subsubsection{Issues with Deconfounding}
Despite one of the fundamental assumptions of the Time Series Deconfounder \cite{bica2020time} being the Deconfounder Assumption from \cite{wang2019blessings}, there has been increasing concern with the validity of causal inference under the deconfounder assumption, including \cite{imai2019deconfounderComments,damour2019deconfounderComments,ogburn2019deconfounderComments,wangBlei2020clarifyingTheDeconfounder,ogburn2020counterexamplesToDeconfounder}.
In light of these concerns, which indicate that the deconfounder assumption alone is insufficient to do full causal inference in the presence of hidden confounders, there are a number of pathways suggested for still utilizing the deconfounder assumption effectively: assuming uniqueness and exact reconstruction of confounders from treatments alone \cite{imai2019deconfounderComments}; further parametric identification assumptions \cite{damour2019deconfounderComments}; and identifiable or sufficiently rich proxy variables \cite{miao2018proxyVariablesUnmeasuredConfounders,tchetgen2020introToProximalCausal,kuroki2014measurement}.
Consequently, the view that applying the Deconfounder Assumption to confounded longitudinal studies directly seemingly faces the majority of these existing issues.
Following up work in the domain of deconfounding observational data in the longitudinal domain has considered different sets of assumptions to recover unbiased identification of effects.
\cite{hatt2021sequential} considers the assumption that there is a confounder that is static over time.
Without requiring the single sequential strong ignorability assumption, they then show strong ignorability under the deconfounding assumption.
Although the deconfounding assumption seems better warranted in the case of static confounders, it is possible that necessitating the static confounders assumption will exclude many potential cases of interest for causal inference on longitudinal data.
\cite{kuzmanovic2021deconfounding} instead considers that the observed covariates might not directly measure unobserved confounders but instead can be considered as noisy measurements of the true confounders.
Unfortunately, no theoretical results are provided and although an appealing notion in the presence of abundant data, often greater care is required in the domain of causal inference.
Both negative controls and proximal causal learning \cite{miao2018proxyVariablesUnmeasuredConfounders,tchetgen2020introToProximalCausal} further delineate between proxy covariates which influence the outcome and proxy covariates which influence the treatment.
In the absence of such distinctions, causal inference may fail to achieve unbiased estimates.
Consequently, it is likely that nonparametric treatment of time-dynamic data will require leveraging emerging work such as negative controls and proximal causal learning.

\subsubsection{Discussion on LipCDE} In this work, we attempt to make practical assumptions toward achieving unbiased estimates for causal inference in the longitudinal setting.
As a consequence, we further existing work on time series deconfounding with an extended assumption combining both the regimes of static confounders and noisy proxy variables.
In this way, we subdivide hidden confounders into two independent regimes: the low-frequency components and the high-frequency components.
After this division, we can treat each uniquely.
The low-frequency confounders represent the static confounders under which reconstruction of the confounding variable using the deconfounder assumption seems plausible after sufficient observation of treatments.
With sufficiently bounded frequency alongside the Lipschitz influence assumption, we can achieve a bounded range of plausible confounders.

Although unlikely to completely debias an estimation procedure, it is possible to significantly reduce the bias and variance in a way dependent on the nuisance parameter of the low-frequency bound $\omega_\ell$.
Such a balance should be considered in the context of the high-frequency confounders, which are then assumed to have sufficient proxy variables available in the observed data.
Afterward, additional care should be taken for these proxy variables in terms of sufficient assumptions for identifiability.
Possible routes include making assumptions directly on the noise distribution or identifying confounders represented by each covariate.
Both of these paths often require greater care, inspection, or prior knowledge of the data which is being studied than is warranted by many practitioners of machine learning. 
Consequently, it is likely that methods developing the work of proximal causal learning in extension to longitudinal data will be required to suit such needs flexibly.
We delay such a study detailing the application of proximal causal learning to dynamic time for future work.

\subsection{Extension of Theorem~\ref{the1} to  continuous time setting}
\label{app_prove}

 Following the same techniques as used in the appendices of ~\cite{bica2020time,wang2019blessings}, we introduce several definitions and lemmas that will help us relocate Theorem~\ref{the1} for discrete-time setting to adapt it to continuous time setting.  As a reminder,  at each timestep $t_k$ under continues-time setting, the random variable $z_{t_k} \in Z_{t_k}$ is constructed as a function of the history path until timestep $t_k: z_{t_k}= g\left({H}_{t_{(k-1)}}\right)$, where ${H}_{t_{(k-1)}}=\left({\mathbf{Z}}_{t_{(k-1)}}, {\mathbf{X}}_{t_{(k-1)}}, {\mathbf{A}}_{t_{(k-1)}}\right)$ takes values in ${\mathcal{H}}_{t_{(k-1)}}={\mathcal{Z}}_{t_{(k-1)}} \times {\mathcal{X}}_{t_{(k-1)}}\times {\mathcal{A}}_{t_{(k-1)}}$ and $g: {\mathcal{H}}_{t_{(k-1)}} \rightarrow \mathcal{Z}$. In order to obtain unbiased estimation using hidden confounders $Z_{t_k}$, the following property needs to hold:
\begin{align}
 Y(a_{\geq t_k}) \perp \!\!\! \perp  (A_{t_{k_1}}, \cdots, A_{t_{k_j}})| X_{t_k}, A_{t_{(k-1)}}, Z_{t_k},
\end{align}

$\forall {a}_{\geq t_k}$ and for all  $t_k$ in irregular samples.

\begin{definition}
\textbf{Continuous  sequential Kallenberg construction}

At timestep $t_k$, we say that the distribution of assigned treatment $\left(A_{t_{k}1}, \ldots A_{t_{k} j}\right)$ admits a continuous sequential Kallenberg construction from the random variables $Z_{t_k}=g\left({H}_{t_{(k-1)}}\right)$ and $X_{t_k}$ if there exist measurable functions $f_{t_{k}j }: \mathcal{Z}_{t_k} \times \mathcal{X}_{t_k} \times[0,1] \rightarrow \mathcal{A}_{t_{k} j}$ and random variables $U_{j t_k} \in[0,1]$ for $j$ assigned treatments, such that:
\begin{align}
\label{eq13.}
    A_{{t_k} j}=f_{{t_k} j}\left(Z_{t_k}, X_{t_k}, U_{{t_k} j}\right),
\end{align}
where $U_{{t_k} j}$ marginally follow Uniform $[0,1]$ and jointly satisfy:
\begin{align}
    U_{{t_k} j} \perp \!\!\! \perp  Y(a_{\geq t_k}) | X_{t_k}, H_{t_{(k-1)}}, Z_{t_k}
\end{align}
for all ${a}_{\geq t_k}$ with $j$ assigned treatments, where $f_{{t_k} j}$ are measurable.
\end{definition}

\begin{lemma}
\label{lemma2}
(Continuous  sequential Kallenberg construction $t \Rightarrow$ Sequential strong ignorability.) If at every timestep $t_k$, the distribution of assigned treatments $A_{{t_k} j}$ admits a continuous  sequential Kallenberg construction from $Z_{t_k}$ and $X_{t_k}$ then we obtain sequential strong ignorability.
\end{lemma}
\begin{proof}
Without loss of generality, we assume that $\mathcal{A}_{j}$ are Borel spaces. For any irregular timestamps $t_k$ we assume $\mathcal{Z}_{t_k}$ and $\mathcal{X}_{t_k}$ are measurable spaces. As $A_{{t_k} j}$ admits  continuous  sequential Kallenberg construction, we have 
\begin{align}
    U_{{t_k} j} \perp \!\!\! \perp  Y(a_{\geq t_k}) | X_{t_k}, H_{t_{(k-1)}}, Z_{t_k}
\end{align}
for all ${a}_{\geq t_k}$ with $j$ assigned treatments. This implies that:
\begin{align}
\left(Z_{t_k}, X_{t_k}, U_{{t_k} j} \right) \perp \!\!\! \perp Y(a_{\geq t_k}) | X_{t_k}, H_{t_{(k-1)}}, Z_{t_k}
\end{align}
Since the $A_{{t_k} j}$ are measurable functions according to Eq.~\ref{eq13.} and $H_{t_{(k-1)}}=\left(X_{t_{(k-1)}}, A_{t_{(k-1)}}, Z_{t_{(k-1)}}
\right)$, we have that sequential strong ignorability holds:
\begin{align}
A_{{t_k} j} \perp \!\!\! \perp  Y(a_{\geq t_k}) | X_{t_k}, A_{t_{(k-1)}}, Z_{t_k}
\end{align}
for all ${a}_{\geq t_k}$ with $j$ assigned treatments of irregular samples.
\end{proof}

\begin{lemma}
\label{lemma3}
 (Factor models for the assigned treatments $\Rightarrow$ Sequential continuous  sequential Kallenberg construction.) Under weak regularity conditions\footnote{Regularity condition: The domains of the causes $\mathcal{A}_{j}$ are Borel subsets of compact intervals. Without loss of generality, we assume $\mathcal{A}_{j}=[0,1]$ for $j$th treatment.}, if the distribution of assigned causes $p\left({\mathbf{a}}_{T}\right)$ can be written as the factor model $p\left(\theta, {\mathbf{x}}_{T}, {\mathbf{z}}_{T}, {\mathbf{a}}_{T}\right)$ then we obtain a continuous sequential Kallenberg construction for irregular  timestamps.
\end{lemma}

The proof for Lemma~\ref{lemma3} uses Lemma $2.22$ in~\cite{kallenberg1997foundations} (kernels and randomization): Let $\mu$ be a probability kernel from a measurable space $S$ to a Borel space $T$. Then there exists some measurable function $f: S \times[0,1] \rightarrow T$ such that if $\vartheta$ is $U(0,1)$, then $f(s, \vartheta)$ has distribution $\mu(s,)$, for every $s \in S$.

\begin{proof}
For timestep $t_k$, consider the random variables $A_{t_{k}j} \in \mathcal{A}_{t_k{j}}, X_{t_k} \in \mathcal{X}_{t_k}, Z_{t_k}=g\left({H}_{t_{(k-1)}}\right) \in \mathcal{Z}_{t_k}$ and $\theta_{j} \in \Theta$. We assume sequential single strong ignorability holds. Without loss of generality, we assume $\mathcal{A}_{t_k j}=[0,1]$ for $j$th treatment.

From Lemma $2.22$ in~\cite{kallenberg1997foundations} , there exists some measurable function $f_{t_k j}: \mathcal{Z}_{t_k} \times \mathcal{X}_{t_k} \times[0,1] \rightarrow[0,1]$ such that $U_{t_k j} \sim$ Uniform $[0,1]$ and:
\begin{align}
A_{t_k j}=f_{t_k j}\left(Z_{k_t}, X_{k_t}, U_{t_k j}\right)
\end{align}
with  $U_{t_k} \perp \!\!\! \perp A_{t_1}$ for all irregular samples. It remains to show that
\begin{align}
    U_{{t_k} j} \perp \!\!\! \perp  Y(a_{\geq t_k}) | X_{t_k}, H_{t_{(k-1)}}, Z_{t_k}
\end{align}
This can be seen by a distinction of cases. For any $j$ treatment: if there exists a random variable $V_{t_k}$ (not equal to $Z_{t_k}$ or $X_{t_k}$ almost surely) that confounds $U_{t_k}$ and $Y\left(a_{\geq t_k}\right)$, it is either (i) time-invariant or (ii) time-varying. (i) If $U_{t_k}$ is time-invariant, then it would also confound $U_{s}$ for $s \neq t_k$, which introduces depedences between the random variables $U_{t_k}$ for all irregular samples. However, since $U_{t_k}$ are drawn $i i d$ from Uniform $[0,1]$, this cannot be the case. Otherwise, $U_{t_k}$ and $U_{s}$ for $s \neq t_k$ would not be jointly independent. (ii) If $V_{t_k}$ is time-varying, then $V_{t_k}$ would confound $A_{t_k}$ through $U_{t_k}$. As a consequence, $V_{t_k}$ would be also a confounders for $A_{t_k}$, which means $V_{t_k} \subseteq U_{t_k}$. As a result, there cannot be another random variable that confounds $U_{t_k}$, and therefore $U_{{t_k} j} \perp \!\!\! \perp  Y(a_{\geq t_k}) | X_{t_k}, H_{t_{(k-1)}}, Z_{t_k}$ holds true.

\end{proof}

\section{Related Work}
\label{REW}
Another different line of research involves the use of synthetic control measures for counterfactual estimation \cite{abadie2010synthetic, chernozhukov2021exact}. \cite{doudchenko2016balancing} use negative weights and intercept terms to estimate weights depending on the data structure. \cite{ding2020dynamical} propose time-varying weights to model the changing correlation between variables, and \cite{athey2021matrix} interpret counterfactual estimation as a matrix completion problem using matrix normalization. However, although the goodness of fit can be improved, matching in discrete time is still difficult with irregularly aligned data, i.e., unit observations that are not aligned in time. To this end, differential equations have been introduced into causal and counterfactual inference. \cite{rubenstein2016deterministic} propose that the equilibrium state of a first-order ODE system \cite{chen2018neural} can be described by a deterministic structural causal model, even with non-constant interventions.  ODE-RNN \cite{rubanova2019latent} moderates the trajectory of interest using irregularly-sampled data. Additionally, \cite{Bellot2021PolicyAU} estimates counterfactual path explicitly using the formalism of controlled differential equations (CDE).  The synthetic control literature is usually discussed in contrast to the structural time series model.  
Besides, some structural methods  rely on the regularity of the treated group trajectories over time to infer counterfactual estimates to  balance the distribution between treated and control groups~\cite{fan2023dish}, whereas synthetic controls rely on the regularity of each group to infer counterfactual estimates  to match treated units to control units \cite{Bellot2021PolicyAU, morrill2021neural}. 
Note that our work is different from another research line of identifying causal structure, i.e., causal discovery, where proposed approaches include: \cite{spirtes2000causation, tian2013causal, peters2014causal, huang2019causal}, etc.

\section{Broader Impact}

Causal analysis is one of the most fundamental problems in time series analysis and easily finds many applications in finance, retail, healthcare, transportation, manufacturing, etc. Among them, estimating the treatment effects for time series data is one key task in many industries and research scenarios which are extremely challenging due to the existence of hidden confounders and dynamic causal relationships of irregular samples or missing observations. 

In this paper, we leverage recent advances in Lipschitz regularization and neural controlled differential equation (CDE) to tackle the above challenges, leading to effective and scalable solutions to causal analysis for time series applications in the wild. LipCDE can directly model the dynamic causal relationships between historical data and outcomes with irregular samples by considering the boundary of hidden confounders given by Lipschitz constraint neural networks.

Although current models including LipCDE are still far away from using causality to figure out the counterfactual world  and to estimate treatment effect absolutely correctly, we do believe that the margin is decreasing rapidly. We would like to highlight that researchers are able to understand and mitigate the potential risks in estimating treatment effects. Especially in the healthcare domain, illogical or unreasonable treatments  would be a disaster for an individual or humanity. In addition, we suggest researchers take a people-centered approach to build the responsible AI system  with the features of fairness, interpretability, privacy, security, and accountability.

\end{document}